\documentclass[twoside]{article}
\usepackage{aistats2020}

\usepackage{microtype}
\usepackage{graphicx}
\usepackage{subfigure}
\usepackage{booktabs} 

\usepackage{hyperref}

\RequirePackage[numbers]{natbib}

\setcitestyle{square}
\usepackage{amsmath}
\usepackage{tabularx}

\usepackage[colorinlistoftodos]{todonotes}

\usepackage{mathtools}
\usepackage{amsthm}
\usepackage{mdframed}

\usepackage{graphicx}

\newcommand{\hA}{\hat{A}}
\newcommand{\hK}{\hat{K}}

\usepackage{dblfloatfix}
\usepackage{amsmath,amssymb,amsthm,xpatch}
\DeclareMathOperator{\sign}{sign}
\DeclareMathOperator{\diag}{diag}
\usepackage{booktabs}  
\usepackage{float}
\newtheoremstyle{mytheoremstyle}{0pt}{0pt}{\itshape}{}{\bfseries}{.}{.5em}{} 
\newtheoremstyle{myplain}
  {1cm plus 1cm minus 0.5cm}
  {1cm plus 1cm minus 0.5cm}
  {\itshape}
  {}
  {\bfseries}
  {.}
  {.5em}
  {}

\newtheorem{theorem}{Theorem}
\newtheorem{lemma}{Lemma}
\newtheorem{claim}{Claim}

\makeatletter
\renewenvironment{proof}[1][\proofname]{\par
  \vspace{-\topsep}
  \pushQED{\qed}%
  \normalfont
  \topsep0pt \partopsep0pt 
  \trivlist
  \item[\hskip\labelsep
        \itshape
    #1\@addpunct{.}]\ignorespaces
}{%
  \popQED\endtrivlist\@endpefalse
  \addvspace{6pt plus 6pt} 
}
\makeatother

\RequirePackage{algorithm}
\RequirePackage{algorithmic}

\begin{document}

\twocolumn[

\aistatstitle{Greedy Frank-Wolfe Algorithm for Exemplar Selection}

\aistatsauthor{ Gary Cheng \And Armin Askari \And Kannan Ramchandran \And Laurent El Ghaoui }

\aistatsaddress{ UC Berkeley \And UC Berkeley \And UC Berkeley \And UC Berkeley } ]

\begin{abstract}
 In this paper, we consider the problem of selecting representatives from a data set for arbitrary supervised/unsupervised learning tasks. We identify a subset $S$ of a data set $A$ such that 1) the size of $S$ is much smaller than $A$ and 2) $S$ efficiently describes the entire data set, in a way formalized via convex optimization. 
 In order to generate $|S| = k$ exemplars, our kernelizable algorithm, Frank-Wolfe Sparse Representation (FWSR), only needs to execute $\approx k$ iterations with a per-iteration cost that is quadratic in the size of $A$. This is in contrast to other state of the art methods which need to execute until convergence with each iteration costing an extra factor of $d$ (dimension of the data). Moreover, we also provide a proof of linear convergence for our method. We support our results with empirical experiments; we test our algorithm against current methods in three different experimental setups on four different data sets. FWSR outperforms other exemplar finding methods both in speed and accuracy in almost all scenarios.
\end{abstract}

\section{Introduction}
\subsection{Overview}
In the areas of computer vision, signal processing and machine learning, it has become important not only to improve the performance of models, but also to be able to train these models efficiently. This has motivated areas like dimensionality reduction that help save on computational resources and memory requirements by compressing the feature space; a non-exhaustive list of techniques include PCA \citep{wold1987}, random projections \citep{candes2006}, generalized discriminant analysis \citep{mika1999}, local linear embeddings \citep{roweis2000} and non-negative matrix factorization \citep{lee1999}.

A related problem is reducing the object-space, or reducing the number of data points in a data set. Exemplar selection is aimed at solving this problem: finding a minimal set of representatives, or \textit{exemplars}, of the data set that effectively represent the rest of the data points. 
These methods can be separated into two groups: wrapper methods and filter methods. The former selects exemplars based on the accuracy obtained by a classifier, whereas the latter approach selects exemplars based on an objective function which is not based on a classifier \citep{olvera2010}. In this paper, we work with filter methods.

\subsection{Paper contribution}
Existing filter methods for exemplar selection are either fast but do not perform well on different learning tasks, or perform well on learning tasks but do not scale well with larger data sets. In this work, we 
\begin{enumerate}
    \item Develop a boolean formulation for the exemplar finding problem
    \item Develop a kernelizable, greedy Frank-Wolfe based algorithm, Frank-Wolfe Sparse Representation (FWSR), to optimize a convex relaxation of the boolean problem
    \item Reduce per-iteration cost of state-of-the-art methods from $\mathcal{O}(n^2d)$ to $\mathcal{O}(n^2)$
    \item Explain the greedy, early termination condition of FWSR
    \item Prove a linear convergence rate for FWSR
\end{enumerate}

Finally, we compare FWSR against other exemplar selection methods in three different experimental setups. 

\section{Related Literature}
The filter method of finding exemplars based on a sparse, auto-regressive model (SMRS) was introduced by \citep{vidal12}. Extensions of this work include Sparse Subspace Clustering (SSC) which uses the learned coefficient matrix as an affinity matrix in spectral clustering \citep{ng01}. SMRS and its variants such as D-SMRS \citep{dsmrs} and Kernelized SMRS \citep{ksmrs} currently attain state-of-the-art results for exemplar selection on different supervised learning tasks. The aforementioned methods use the Alternating Direction Method of Multipliers (ADMM) to solve an optimization problem that requires a one-time inversion of a dense matrix, as well as dense matrix multiplications at every iteration. Even with the state-of-the-art speed improvements applied to SMRS, the \textit{per-iteration} and up-front cost is still $\mathcal{O}(n^2 d)$ \citep{pourkamali2018efficient}, making these methods unsuitable for even moderately-sized data sets. \citep{vidal16} try to address this concern by introducing a greedy Orthogonal Matching Pursuit relaxation of SSC. However, in doing so, they remove the group lasso penalty and shift their focus from exemplar selection to clustering.

The auto-regressive formulation of exemplar selection can be thought of as a version of dictionary learning. Methods like K-SVD \citep{ksvd2006} attempt to solve the regression problem 
\begin{align*}
    \min_{D, X} \left\|A - DX \right\|_{F}^2 :  \forall i,\ \left\|X_{(i)}\right\|_{0} \leq k,
\end{align*} 
where $A$ is the data matrix, and $X_{(i)}$ represents the $i$th column of $X$. In the setting of exemplar selection, we restrict the dictionary $D$ to be the data matrix $A$. SMRS and other similar works \citep{sapiro12} can be seen as solving this particular instance of dictionary learning. Note that in K-SVD, simply replacing $D$ by $A$ generates the trivial solution $X = I$, motivating the introduction of the group lasso constraint.

An instance of exemplar selection that is not formulated as an auto-regressive optimization is $k$-medoids \citep{kmedoid1987}. Unlike $k$-means, $k$-medoids  requires that the centers of the clusters be data points, which can be treated as exemplars of the $k$ classes. However, $k$-medoids in general does not converge to the global optimum and does not necessarily cluster points lying on the same subspace.

There are other indirect methods whose solutions can be interpreted in the context of exemplar selection. For instance, Rank Revealing QR Decomposition (RRQR) \citep{rrqr1992} selects data points based on a permutation matrix of the data which gives a well conditioned submatrix. The Column Subset Selection Problem (CSSP) is also related to selecting exemplars. The problem is to identify $k$ columns of a matrix $A$, called $C$, which minimize $\left\|A - P_C A\right\|_F$ where $P_C$ is the projection operation onto $C$. Other ways of addressing this problem include randomized sketching methods like CUR decomposition \citep{drineas08}; \citep{boutsidis09} analyze a variant that combines ideas from CUR decomposition with RRQR.

There is also another body of work related to exemplar finding called coreset construction. Coreset construction is in the same spirit as exemplar selection and has had recent success in the context of PCA and $k$-means \citep{feldman13, feldman16}. Despite this, the aforementioned coreset algorithms are wrapper methods, and it is unclear how to generalize their construction to arbitrary learning problems \citep{campbell17}. We instead focus on filter methods, which are problem-agnostic.

In this paper, we employ the Frank-Wolfe algorithm \citep{frank1956algorithm} for constructing our set of exemplars. Although introduced in the optimization community over half a century ago, the Frank-Wolfe algorithm (also known as the conditional gradient algorithm) has experienced renewed interest in recent years due to its vast applications in machine learning \citep{lacoste2013affine}. In particular, the algorithm is a greedy one that for certain problem formulations results in sparse iterates and solutions which is widely applicable to the sparse learning community. Although there are different variants of the algorithm based on the one originally proposed, these other methods make assumptions that do not fit the setup of our problem (see Remark 2 of Section \ref{sec:rate}). In the context of exemplar selection, \citep{clarkson2010coresets} marry the ideas of coreset construction and the Frank-Wolfe algorithm. Specifically, the authors sharpen bounds on coreset construction and algorithmic convergence rates for canonical machine learning problems. However, their results pertain to problems that can be written as the maximization of a concave function over the simplex. In this work, we instead opt to work with a group lasso domain and show that our formulation is indeed a relaxation of the natural boolean problem. Furthermore, we are able to obtain a linear convergence rate for our problem while the Frank-Wolfe algorithm for an arbitrary problem only produces a sublinear convergence rate \citep{frank1956algorithm}.

\section{Problem Formulation}
\subsection{Notation}
Let $\left\| \cdot \right\|_{F}$ be the Frobenius norm. Let $X^{(i)}$ and $X_{(i)}$ denote the $i$-th row and column of a matrix $X$ respectively; $X_{ij}$ denotes the $(i,j)$th entry of a matrix $X$. Let $X_t$ denote the value of $X$ on the $t$th iteration. Let $e_j$ denote the $j$th standard basis vector. For $q\geq1$, we refer to $\sum_{i=1}^{n}\|X^{(i)}\|_{q}$ as the ``$q$-norm group lasso''. We use $\mathcal{M}_{q, \beta}$ to denote the the $q$-norm group lasso ball of radius $\beta$. We denote our feature matrix as $A \in \mathbb{R}^{d \times n }$ where each column represents a data point in $d$-dimensional space. $k$ refers to the number of desired exemplars. We define the Gram matrix $K \coloneqq A^\top A$. Finally, $\textbf{1}$ denotes a vector of ones of appropriate dimension.

\subsection{Objective}
\subsubsection{Boolean Selection Problem}
We formulate exemplar selection as a boolean selection problem:
\begin{align*}
\min_{u \in \{0,1\}^n}
\sum_{i=1}^n  \lambda  u_i + 
\min_{x_i} \| A \diag(u)x_i - a_i\|_2^2 + \rho^2 \|x_i\|_2^2 
\end{align*}
This objective uses the boolean vector $u$ to select a subset of data points (exemplars) which are most cost-efficient in representing the entire dataset via ridge regression. The hyperparameters $\rho$ and $\lambda$ control the ridge term and sparsity of $x$ and $u$ respectively.
We can rewrite the objective in matrix form:
\begin{align}\label{eq:bool}
\phi \coloneqq
\min_{u \in \{0,1\}^n} 
\lambda 1^\top u +\min_X  \left\|A \diag(u) X - A\right\|_{F}^2  + \rho^2 \left\|X\right\|_{F}^2 
\end{align}
\subsubsection{Convex Relaxation}
From \eqref{eq:bool}, it is clear that $u_j = 0 \iff X^{(j)} = 0$. This implies $u_j = \mathbf{1}(X^{(j)})$, where $\mathbf{1}(v) = 0$ if $v$ is the $0$ vector and 1 otherwise. We use this fact to make the following relaxation:
\begin{align*}
\rho^2 \|X^{(j)}\|_{2}^2 &+ \lambda \mathbf{1}(X^{(j)}) \geq 2\lambda {\mathcal B}\left(\frac{\rho \|X^{(j)}\|_{2}}{\sqrt{\lambda}}\right)\\
&= \left\{ \begin{array}{ll}
2\rho \sqrt{\lambda} \left\|X^{(j)}\right\|_{2} & \mbox{if } \left\|X^{(j)}\right\|_{2} \leq \sqrt{\lambda}/\rho, \\
\rho^2\left\|X^{(j)}\right\|_{2}^2 +\lambda  & \mbox{otherwise.}
\end{array}
\right.\\
&\geq 2 \rho \sqrt{\lambda} \|X^{(j)}\|_{2}
\end{align*}
We call $\mathcal{B}$ the \textit{reverse Huber function}. By repeating this relaxation for $j = 1 \ldots n$, the boolean constraints relax to form the group lasso penalty:
\begin{align*}
\phi \geq \min_X \left\|AX - A\right\|_{F}^2 + 2 \rho \sqrt{\lambda} \sum_{i=1}^{n} \|X^\top e_i\|_{2}
\end{align*}
It should be noted, that while the reverse Huber function is a tighter relaxation, we instead directly use the $l_2$ norm, as it will lead to sparse updates in the FWSR algorithm highlighted in Section \ref{sec:contribute}. Additionally, to fit the framework of Frank-Wolfe, we use the equivalent constrained version of the problem where we have a group lasso constraint instead of a regularization term. With these changes, and the addition of a penalty on translational invariance of $A$, the new training problem becomes
\begin{align} \label{eq:opt}
\min_X f(X) \coloneqq \: \min_X\: &\|A X - A\|_{F}^2 + \eta^2 \|X^\top \textbf{1} - \textbf{1}\|_{2}^2\\
\text{s.t. } &\sum_{i=1}^{n}\|X^{(i)}\|_{q} \leq \beta \nonumber
\end{align}
where $\beta, q, \eta$ are hyperparameters. Intuitively, \eqref{eq:opt} identifies a sparse subset of the data points that best span (i.e., represent) the entire data set. \eqref{eq:opt} can be alternatively viewed as a convex relaxation of the dictionary learning problem, where the dictionary is the data set itself. In this setup, suppose we solve for $X^* = \arg \min_X f(X)$. Then the data points $A^{(j)}$ such that $X_{(j)} \neq 0$ are our exemplars. For the remainder of the paper, we will use the term ``data points corresponding to the non-zero rows of $X^*$" to describe the selected exemplar data points. The row-sparsity (i.e., number of non-zero rows) of $X^*$ is controlled by our choice of  $\beta$ and $q$. Empirically, we have found that $q=2$ generally performs the best. The notion of translational invariance was originally introduced as a constraint ($\textbf{1}^\top X = \textbf{1}^\top$) by \citep{vidal12}\footnote{If $AX = A$ for some $X$, then $\forall z\in \mathbb{R}^d$, $\textbf{1}^\top X = \textbf{1}^\top \iff (A + z\textbf{1}^\top)X = A + z\textbf{1}^\top$}; here we use the $\eta$ hyperparameter to from a relaxed, penalized version of the constraint in order to make our algorithm simpler. It should be noted that the primary motivation behind adding the $\eta$ penalty is due to better observed empirical results. For simplicity, for the remainder of the paper, we rewrite the translation invariance penalty by implicitly augmenting the $A$ matrix with the row $\eta \textbf{1}^\top$. FWSR is also amendable to an optional non-negativity $X \geq 0$ constraint, which is often used with image or text data sets since it has real life interpretations \citep{sapiro12}. We omit it in \eqref{eq:opt} because empirically in our experiments, it did not provide a noticeable benefit.


\section{Contributions}
\label{sec:contribute}
We propose a greedy Frank-Wolfe algorithm, Frank-Wolfe Sparse Representation (FWSR), for solving \eqref{eq:opt} that is faster than other exemplar selection methods and whose selected exemplars enjoy higher test set accuracies when trained on a variety of data sets. The pseudocode of FWSR is displayed in Algorithm \ref{sub1}. Unlike other methods, FWSR will not necessarily solve \eqref{eq:opt} to convergence. 
Rather, FWSR starts with no exemplars and greedily selects one exemplar at every iteration, resulting in the algorithm either 1) greedily terminating as soon as $k$ rows are non-zero and subsequently returning the corresponding exemplars or 2) converging to a $r < k$ row-sparse solution and return those $r$ indices. For either of these cases, we provide Theorem \ref{thm:linconv}, which characterizes a linear convergence rate of FWSR, which means the iterations needed before termination is not large.

\subsection{Algorithm Description}
\label{sec:algdes}
Recall that Frank-Wolfe is a projection-free algorithm that optimizes an objective over a closed, convex set by moving towards the minimizer of its linear approximation at each iteration. Frank-Wolfe is comprised of the following steps: 1) calculate gradient 2) solve linear minimization oracle (LMO) to find descent direction 3) calculate optimal step size via exact line search and 4) repeat until terminating condition. Below, we walk through the steps for FWSR.

\subsubsection{Preprocessing}
We center $A$ columnwise (i.e., vertically; center along datapoints) to replace the need for an explicit bias term. Then, we form $K = A^\top A$ in $\mathcal{O}(n^2d)$ time since it is used frequently later in the algorithm. Note that $K$ can be replaced with any kernel matrix $\Phi(A)$ since FWSR relies only on the Gram Matrix $K$ and not on $A$ directly. We refer to the kernelized variant of Algorithm \ref{sub1} as K-FWSR.
Next, we initialize $X_0$ to the $0$ matrix which represents having selected no exemplars at the start of the algorithm. 


\subsubsection{Gradient Calculation}
The gradient of our objective function is $2KX -2K$. Even though $K$ is calculated once, explicitly calculatig the gradient at each iteration is expensive (naively $\mathcal{O}(n^3)$) due to the matrix-matrix product. Because of the structure of the problem, we are able to efficiently calculate $(KX)_t$. As explained in subsection \ref{sec:step}, $X_t$ is a weighted average of $X_{t-1}$ and a rank-1 matrix $S_{t-1} = e_j v^\top$ for some $j$ and vector $v$. This implies that $(KX)_t$ is a weighted average of $(KX)_{t-1}$ and $K S_{t-1} = K_{(j)}v^\top$. Since we know $(KX)_{t-1}$ at step $t$, we can calculate $(KX)_{t}$ and the entire gradient in $\mathcal{O}(n^2)$ time as shown in lines \ref{grad_update} and \ref{grad_calc} of Algorithm \ref{sub1}. 

\subsubsection{LMO Calculation}
\label{sec:lmo}
With the gradient formed, we then solve the LMO, which specifies the direction of descent. In general, the LMO recovers a direction $S_t$ to take the next step, specified by the following optimization problem:
\begin{align*}
S_t = LMO(\nabla f(X_t)) = \arg\min_{S' \in \mathcal{M}} \langle S', \nabla f(X_t)\rangle
\end{align*}
In the case when the vertices of the feasible set are sparse, the Frank-Wolfe algorithm produces sparse iterates. Because of the group lasso constraint in \eqref{eq:opt}, the solution matrix $S_t$ will be a rank-1 matrix with only one non-zero row $s$ at index $j = \arg\max_i \|\nabla f(X_t)^{(i)}\|_{p}$, where $p$ corresponds to the dual-norm of $q$. The magnitude of the non-zero row of $S_t$ will be $\beta$ (i.e., $\|S_t^{(j)}\|_{q} = \beta$), and the direction of the row will be chosen to minimize the inner product. Depending on the value of $q$, we can change lines \ref{max_lmo} and \ref{lmo} of Algorithm \ref{sub1} accordingly. Specifically, line \ref{lmo} becomes:
\begin{enumerate}
  \item if $q = 1$, then $S_t^{(j)}$ will be all zero except the entry corresponding to the largest magnitude value in $\nabla f(X_t)^{(j)}$; this entry will have value $-\beta$
  \item if $q = 2$, then $S_t^{(j)}$ will be of the form $-\beta \nabla f(X_t)^{(j)} / \|\nabla f(X_t)^{(j)}\|_{2}$
  \item if $q = \infty$, then $S_t^{(j)}$ will have the form such that the $l$th entry, $s_l = -\beta \sign(\nabla f(X_t)_{jl})$ 
\end{enumerate}

Now with $S_t$ solved for, we can explicitly form the direction of descent, $D_t \coloneqq S_t - X_t$. 

\subsubsection{Step Size Calculation}
\label{sec:step}
Once the LMO is solved, the next iterate is calculated by performing exact line search along the direction between the previous iterate and $D_t$. More explicitly, the optimal step size $\gamma_t$ at iteration $t$ is
\begin{align*}
\gamma_t = \arg \min_{\gamma'} f(X_t + \gamma' D_t)
\end{align*}

For \eqref{eq:opt}, $\gamma_t$ has a closed form expression. To calculate $\gamma_t$ in $\mathcal{O}(n^2)$ time, we expand the numerator and denominator of the value in line \ref{opt_step} of Algorithm \ref{sub1}. We independently calculate each component and use the fact that the trace of matrix products can be performed in quadratic time. 

Since at most one extra row ($j$th row) becomes non-zero via a rank one perturbation (the addition of $\gamma_t S_t$), line 17 can be interpreted as the algorithm selecting the $j$th data point as an exemplar. 

\subsubsection{Repetition \& Terminating Conditions}
\label{sec:repeat}


We repeat this three step (gradient calculation, LMO, step size calculation) procedure until either \textbf{NumExemplars}$(X_t)$, the row-sparsity of the iterate $X_t$, is equal to the number of desired exemplars $k$ or we converge (line \ref{term_thres} is satisfied). At this point, the algorithm terminates with \textbf{PickExemplars}$(A, X_t, k)$ returning the columns of $A$ that correspond to the non-zero rows of $X_t$. 

\begin{algorithm}[t]
    \caption{Frank-Wolfe Sparse Representation ($q = 2$)} 
    \label{sub1}
    \begin{algorithmic}[1]
    \STATE {\bfseries Input:} $A \in \mathbb{R}^{d, n}, k, \beta, \eta$
    \STATE center $A$ column-wise \& augment $A$ with row $\eta \textbf{1}^\top$
    \STATE $K = A^\top A$
    \STATE $X_0, S_0, E, \gamma_0, j, t= 0, 0, 0, 0, 0, 1$
    \STATE $(KX)_0 = 0$
    \WHILE {$E < k$}
        \STATE $(KX)_t = (1 - \gamma_{t-1}) (KX)_{t-1} + \gamma_{t-1} K_{(j)}S_{t-1}^{(j)}$ \label{grad_update}
        \STATE $\nabla f_t = 2(KX)_t - 2K$ \label{grad_calc}
        \STATE $j = \arg \max_i \|(\nabla f_t)^{(i)}\|_{2}$ \label{max_lmo}
        \STATE $S_t = 0$
        \STATE $S_t^{(j)} = -\beta (\nabla f_t)^{(j)}/\|(\nabla f_t)^{(j)}\|_{2}$ \label{lmo}
        \STATE $D_t = S_t - X_t$ 
        \IF {$- \langle \nabla f_t, D_t \rangle< \delta$} \label{term_thres}
            \STATE \textbf{break}
        \ENDIF
            \STATE $\gamma_t = \min\left( 1, \frac{Tr(D_t^\top (K - KX_t))}{D_t^\top K D_t} \right)$ \label{opt_step}
            \STATE $X_{t+1} = X_t + \gamma_t D_t$
            \STATE $E = $\textbf{NumExemplars}$(X_{t+1})$
            \STATE $t = t + 1$
    \ENDWHILE
    \STATE \textbf{Return: }\textbf{PickExemplars}$(A, X_t, k)$
    \end{algorithmic}
\end{algorithm}

\subsection{Convergence Rate} \label{sec:rate}
We now present a convergence rate for FWSR.
\begin{theorem}\label{thm:linconv}
 

For $\beta > n$, 
the iterate generated by Algorithm 1, $X_t$, satisfies
\begin{align}
    \|AX_t - A\|_F^2 \leq \dfrac{4\beta^2C^2 \lambda_{\max}(K)\nu^2 }{C^2 \nu^{2 - 2t} + t} .\nonumber
\end{align}
where $C > 0, \nu \in (0,1)$ are constants. For sufficiently large $t$, the convergence rate is linear.

For $\beta \leq n$, the convergence rate is in general sublinear as shown in \citep{frank1956algorithm}.
\end{theorem}


The proof of Theorem \ref{thm:linconv} can be found in \ref{apx:proof}. This theorem shows that FWSR will not need to run for a large number of steps, no matter whether FWSR converges to a solution that has a row-sparsity which is greater than $k$ (i.e., FWSR greedily terminates) or less than $k$ (i.e., FWSR meets gradient convergence condition in line \ref{term_thres}).

\paragraph{Remark 1} One might notice that the choice of $\beta > n$ allows for the trivial solution, the identity matrix. This leads to $n$ exemplars. However, due to the greedy nature of FWSR, the value provided by FWSR lies in which $k \ll n$ rows become non-zero first upon the execution of the algorithm and not necessarily on the ultimate solution to which the algorithm converges to.

\paragraph{Remark 2}
We do not use other variants of Frank-Wolfe (Away-step, Pairwise) because they require storing and cycling through a growing active set of vertices.  The 2-norm group lasso ball has an infinite number of vertices and the $\infty$-norm group lasso ball has number of vertices exponential in $d$, making the potential size of the active set too large. While the $1$-norm group lasso ball only has a finite number of vertices linear in the dimension, it enforces element-wise sparsity as opposed to $q=2,\infty$ which impose row-sparsity. Moreover, $q=2$ empirically outperforms $q=1,\infty$. Another reason we do not use Pairwise Frank-Wolfe is that its analysis is contingent on the domain being the closed convex-hull of a \textit{finite} number of vertices. This makes it incompatible with the 2-norm group lasso ball.

\paragraph{Remark 3} The primary value of Theorem \ref{thm:linconv} is for the case when the objective is not strongly convex, as previous results from \citep{guelat86} already demonstrate a linear convergence rate for the strongly-convex setting when $\beta > n$. Notice that the non-strongly convex setting (e.g., $n > d$) is of interest for our application. Theorem \ref{thm:linconv} ensures that we still have a linear convergence rate even when the number of data points is much larger than the dimension. It should be noted that \citep{lacoste2015global} have linear convergence results for functions of the form $g(Ax) + b^\top x$ where $g(\cdot )$ is a strongly convex function. However, this result is only for Frank-Wolfe variants with polytope-domains. It is of independent interest to find whether the aforementioned results of \citep{guelat86} can be adapted to functions of the form $g(Ax) + b^\top x$.



\subsection{Computational \& Space Complexity}
FWSR requires an up-front cost of $\mathcal{O}(n^2 d)$ to form the Gram Matrix. Each subsequent sparse iteration takes $\mathcal{O}(n^2)$ time to execute as explained in Section \ref{sec:algdes}. Moreover, due to the greedy property of FWSR, the algorithm terminates either when $k$ exemplars have been selected (i.e., when $X_t$ has $k$ non-zero rows), or the error is below a certain threshold (line \ref{term_thres}). As stated earlier, Theorem \ref{thm:linconv} proves that in either case, the number of iterations will not be large. With respect to the amount of storage needed, FWSR's sparse iterates means that we only need to keep track of at most $k$ non-zero rows in $X$. Hence the space complexity of FWSR, excluding the kernel matrix, is $\mathcal{O}(kn)$. 

SMRS and its variants are able to attain state of the art results on different supervised learning tasks using ADMM. In addition to also requiring the $\mathcal{O}(n^2 d)$ calculation of the Gram Matrix $K$, ADMM requires a dense matrix inversion as a preprocessing step and a dense matrix multiplication in every subsequent iteration, resulting in a $\mathcal{O}(\min(n^3, n^2 d))$ cost per iteration using the results of \citep{pourkamali2018efficient}. Additionally, since the iterates generated by ADMM are not necessarilly sparse, iterating until convergence as well as tuning the sparsity hyperparameter/terminating tolerance is necessary. These are expensive requirements that drastically hurt the performance of SMRS. With respect to space complexity, the potential of dense iterates implies that there could be iterations where all $n$ rows of the variable matrix could be non-zero, which results in a $\mathcal{O}(n^2)$ space complexity in addition to the kernel matrix. 

\begin{table}
\centering
\begin{tabular}{ lll }
 \toprule
 &  \textsc{Computation} & \textsc{Memory}\\
 \midrule
 FWSR   & $n^2d + \min(k^\dagger, T) n^2$    & $n^2 + kn$\\
 SMRS&    $n^2d +  T\min(n^2 d, n^3)$ & $n^2 + n^2$   \\
\bottomrule
\end{tabular}
 \caption{Computation term is comprised of complexity to calculate Gram Matrix plus time it takes to iterate to the solution. Memory term is comprised of space for Gram Matrix plus the space for variable matrix. $T$ is the number of iterations required for convergence. $k$ is the number of exemplars desired. $k^\dagger$ is the number of iterations until $X$ is $k$-row-sparse; empirically, $k^\dagger \approx k$.}
\label{tbl:compare}
\end{table}

A comparison of computational and space complexity of FWSR and SMRS is summarized in Table \ref{tbl:compare}. It is clear that with respect to per-iteration cost, number of iterations required, and space complexity, FWSR is able to outperform SMRS. In the following section, we support our claims with empirical results.


\section{Empirical Results}
We compare FWSR against 4 different data reduction methods (random subset selection, SMRS, $k$-medoids, and RRQR)\footnote{We do not consider D-SMRS and Kernelized SMRS since both introduce additional hyperparameters which, coupled with their runtimes on relatively larger data sets, make cross validation very computationally intensive.} in three different experimental setups: two in an unsupervised setting and one in a supervised setting. The first experiment 1) randomly generates $30$ exemplars we wish to recover, 2) takes random convex combinations of these exemplars to generate a data set, and 3) uses exemplar selection methods to recover the generated exemplars from 1). 
The second experiment is on a synthetic Gaussian data set of $k$ artificial clusters. We measure how well each exemplar selection algorithm was able to recover one point from each cluster.
In the last experiment, we compare the algorithms on downstream classification tasks for labeled data sets.   

The Matlab code we use for SMRS is taken from \citep{vidal12}. We modify this code using the matrix inversion lemma as seen in \citep{pourkamali2018efficient}. We use \citep{kmed} implementation of the $k$-medoids algorithm, and we use scipy to implement RRQR. We coded FWSR in python.

In FWSR, the effect of $\beta$ is highly dependent on the number of data points, $n$. In an effort to disentangle this dependency, we parameterize $\beta$ as $n/\alpha$ where $\alpha$ is a hyperparameter that we choose; typically $\alpha \in [0.5,50]$. Additionally, to enforce that SMRS selects no more than $k$ exemplars, we choose the data points corresponding to the $k$ largest $\ell_2$ norm rows of the returned coefficient matrix $X$ as exemplars as explained in \citep{vidal12}.

\subsection{Experiment 1 - Random Convex Combinations}
To quantitatively demonstrate the performance of our algorithm, we first generate 30 exemplar data points in $\mathbb{R}^{200}$. Then, we generate 120 additional data points by repeatedly randomly selecting 3 exemplar data points and performing a random convex combination of these points and adding mean zero Gaussian noise. We then use FWSR, SMRS, RRQR, and $k$-medoids to recover 30 candidate exemplars. Figure \ref{fig:synthetic_rcc} plots the average fraction of exemplars recovered over 10 trials against the standard deviation of the noise. It is clear from the figure that FWSR is not only fast, but also the only method that is unaffected by increasing noise.

\begin{figure*}[t]
    \centering
    \includegraphics[width=0.8\linewidth]{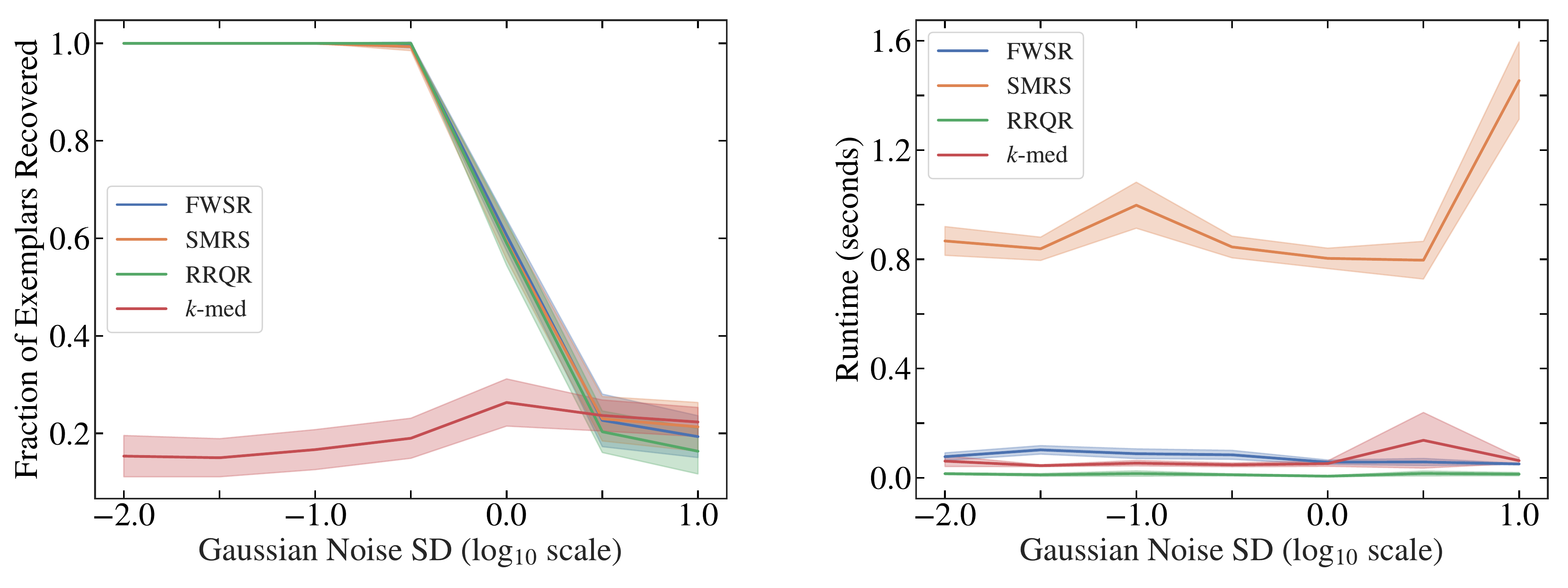}
    \caption{\textbf{Experiment 1} Average fraction of generating exemplars recovered versus level of zero-mean, iid Gaussian noise applied for FWSR, SMRS, RRQR, and $k$-medoids. The shaded regions represent two standard deviations over 10 experiments per noise level. The hyperparameters for FWSR and SMRS are tuned for each level of noise. $k$-medoids, FWSR, and RRQR perform increasingly faster (in that order) and have runtimes on the order of $0.01$s. FWSR, RRQR, and SMRS all perform very similar recovery rate.}
    \label{fig:synthetic_rcc}
\end{figure*}

\subsection{Experiment 2 - Gaussian Clusters}
For this experiment, we generate $1000$ data points and disperse them evenly between $k$ Gaussian clusters in $\mathbb{R}^{1500}$ with covariance $\Sigma = 20^2 I$ using sklearn's \texttt{make\_blob} function. We then use FWSR, SMRS, RRQR, and $k$-medoids to find $k$ exemplars from these $k$ clusters. Figure \ref{fig:synthetic} plots the fraction of the $k$ clusters that were recovered. Without any hyperparameter tuning, we set the sparsity hyperparameter $\alpha = 20$ for SMRS, which is in the range recommended by the authors, and $(\alpha,\eta) = (10,0)$ for FWSR. In Figure \ref{fig:synthetic}, it is clear that with only a few clusters, FWSR is able to recover exemplars from a large percentage of unique clusters compared to the other methods. Although not shown, when the magnitude of the covariance is lowered, both FWSR and RRQR are able to recover the number of clusters with nearly 100\% accuracy while SMRS and $k$-medoids had a recovery rate around 70\%.


\begin{figure}[h]
    \includegraphics[width=1\linewidth]{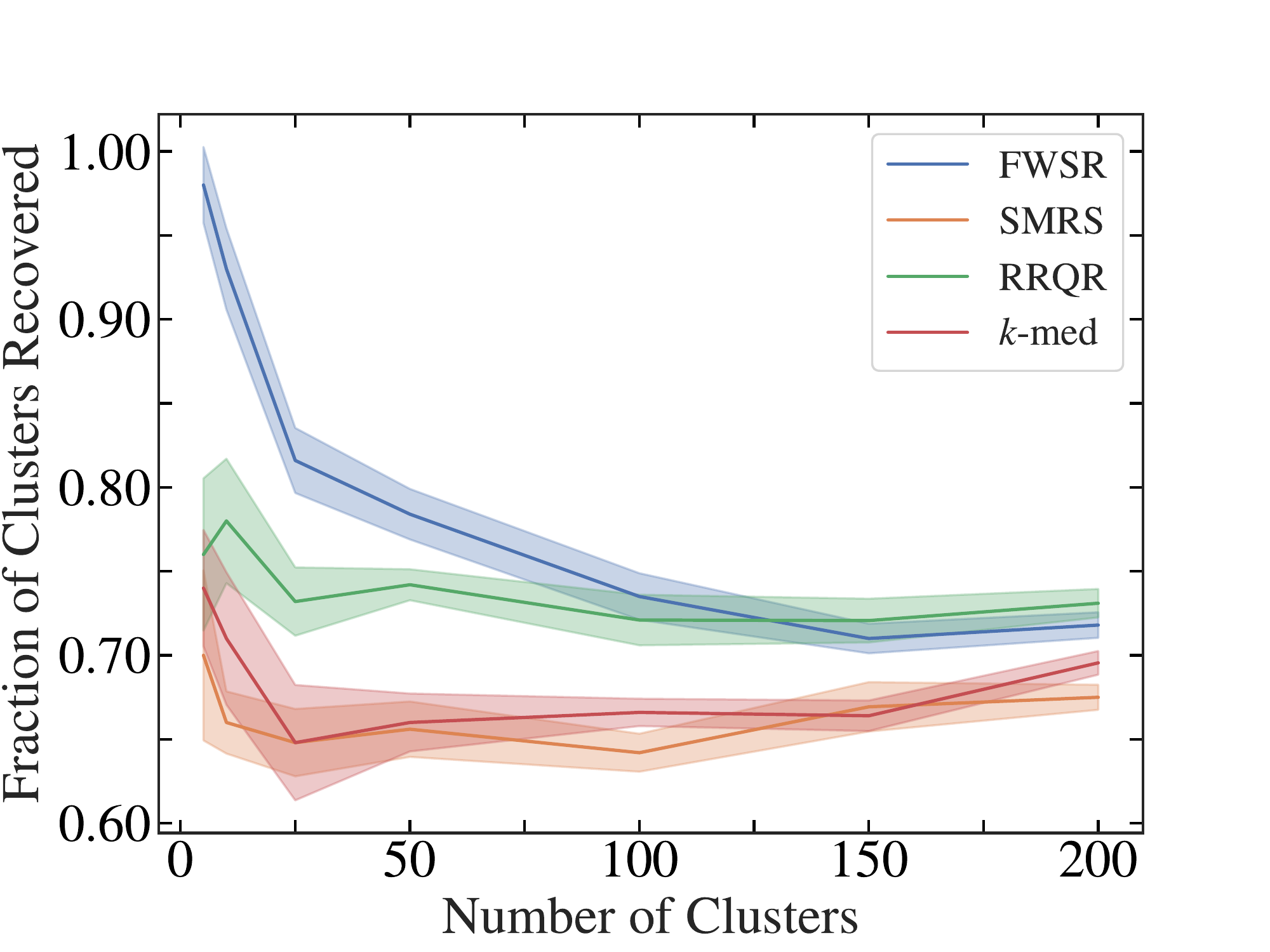}
    \caption{\textbf{Experiment 2} Average fraction of cluster centers recovered versus number of clusters on isotropic Gaussian data for FWSR, SMRS, RRQR, and $k$-medoids. The shaded regions represent one standard deviation over 10 experiments per cluster.}
    \label{fig:synthetic}
\end{figure}

\subsection{Experiment 3 - Labeled Data sets}
Next, we compare FWSR and K-FWSR with a RBF kernel against random subset selection, SMRS, $k$-medoids, and RRQR on downstream tasks. More specifically, for each of the $m$ classes in a labeled training data set, the algorithms select $k$ exemplars per class. These $mk$ exemplars are then used to train a classifier. The exemplar selection algorithms are then compared against one another based on the end-to-end data reduction and training time and validation accuracy. We consider 3 different classifiers: Balanced Linear Support Vector Machines (SVM), $k$-Nearest Neighbors ($k$-NN), and Multinomial Naive Bayes (MNB) all implemented using scikit-learn.  

We cross validate the hyperparameters of the exemplar selection methods and the classifier by comparing validation accuracies for the exemplar-trained classifiers. We repeat this process over (nearly) all combinations of data set, exemplar selection algorithm, and classification model. For non-deterministic methods such as random subset selection and $k$-medoids, we run the exemplar finding algorithm $20$ times and average our results, optimizing hyperparameters for each run.


\begin{table}
\centering
\resizebox{\columnwidth}{!}{
\begin{tabular}{ lccccc}
 \toprule
 Data Set&  \# Class & Train $n$ & Valid. $n$ & $d$ & $k/$class\\
 \midrule
 \textsc{E-YaleB} &  $38$& $1,938$ & $476$ & $1,024$ & $7$\\
 \textsc{News20} &  $20$& $11,314$& $7,532$ &$50,000^*$ & $50$  \\ 
 \textsc{Credit} &  $2^\dagger$& $4,394$ &$1,098$ & $29$ & $10$ \\
 \textsc{EMNIST} &  $62$& $253,523$ &$116,323$ & $784$ & $10$\\
\bottomrule
\end{tabular}
}
 \caption{\textbf{Experiment 3} A description of the data sets used. $*$For \textsc{News20} with \textsc{k-NN}, we use \texttt{feature\_selection.chi2} to reduce it to 5000 dimensional data. $\dagger$The \textsc{Credit} fraud class (492) is much smaller than the non-fraud class (5000); hence we use F1-score and only do exemplar selection from non-fraud class. }
\label{tbl:datasets}
\end{table} 

\begin{table*}[ht]
\centering
\scalebox{1}{\begin{tabular}{lcccccccc}
\toprule
&\multicolumn{2}{c}{\textsc{E-YaleB}} &  \multicolumn{3}{c}{\textsc{News20}} & \multicolumn{2}{c}{\textsc{Credit}} &  \multicolumn{1}{c}{\textsc{EMNIST}} \\
\cmidrule(lr){2-3} 
\cmidrule(lr){4-6} 
\cmidrule(lr){7-8}
\cmidrule(lr){9-9}
 & \textsc{SVM} & \textsc{$k$-NN} & \textsc{SVM} & \textsc{$k$-NN} & \textsc{MNB} & \textsc{SVM} & \textsc{$k$-NN} & \textsc{SVM} \\
 \midrule 
\textsc{All} & $0.994$ & $0.773$ & $0.703$ & $0.266$ & $0.703$ & $0.892$ & $0.911$ & $0.639$ \\ \midrule
\textsc{Random} & $0.811$ & $0.412$  & $0.550$ & $0.213$  & $0.541$ & $0.835$ & $0.284$  & $0.344$  \\
\textsc{FWSR} & $0.903$ & $0.473$ & $\mathbf{0.601}$ & $\mathbf{0.343}$ & $\mathbf{0.625}$ & $0.885$ &$ 0.182$ & $\mathbf{0.515}$ \\
\textsc{K-FWSR} & $0.824$  & $\mathbf{0.515}$ & $0.584$ & $0.305$ & $0.618$ & $\mathbf{0.887}$ & $\mathbf{0.876}$ &  $0.400$\\
\textsc{$k$-med} & $0.851$  & $0.480$  & $0.567$  & $0.162$  & $0.566$  & $0.848$ & $0.434$  & $0.461$ \\
\textsc{RRQR} & $\mathbf{0.908}$ & $0.376$ & $0.375$ & $0.313$ & $0.404$ & $0.557$ & $0.165$ & $0.241$ \\
\textsc{SMRS} & $0.876$ & $0.456$ & $0.568$ & $0.274$ & $0.576$ & $0.812$ & $0.167$ & \; -- \\
\bottomrule
\end{tabular}}
\caption{ \textbf{Experiment 3} Accuracies for different exemplar selection algorithms using different training algorithms on 4 different data sets. We select $7$ exemplars/class for E-YaleB, $50$ exemplars/class for New20, $10$ exemplars in the non-fraud class for Credit, and $10$ exemplars/class for EMNIST, corresponding to 13.7\%, 8.8\%, 0.2\%, and 0.2\% of the data sets respectively. SMRS is not capable of running efficiently on the EMNIST data set due to its large size, so we use -- as a placeholder. Bolded numbers in each column denote the best accuracy attained among all exemplar finding algorithms.}
\label{sup_acc}
\end{table*}

\begin{figure*}[ht]
      \centering
      \includegraphics[width=0.8\linewidth]{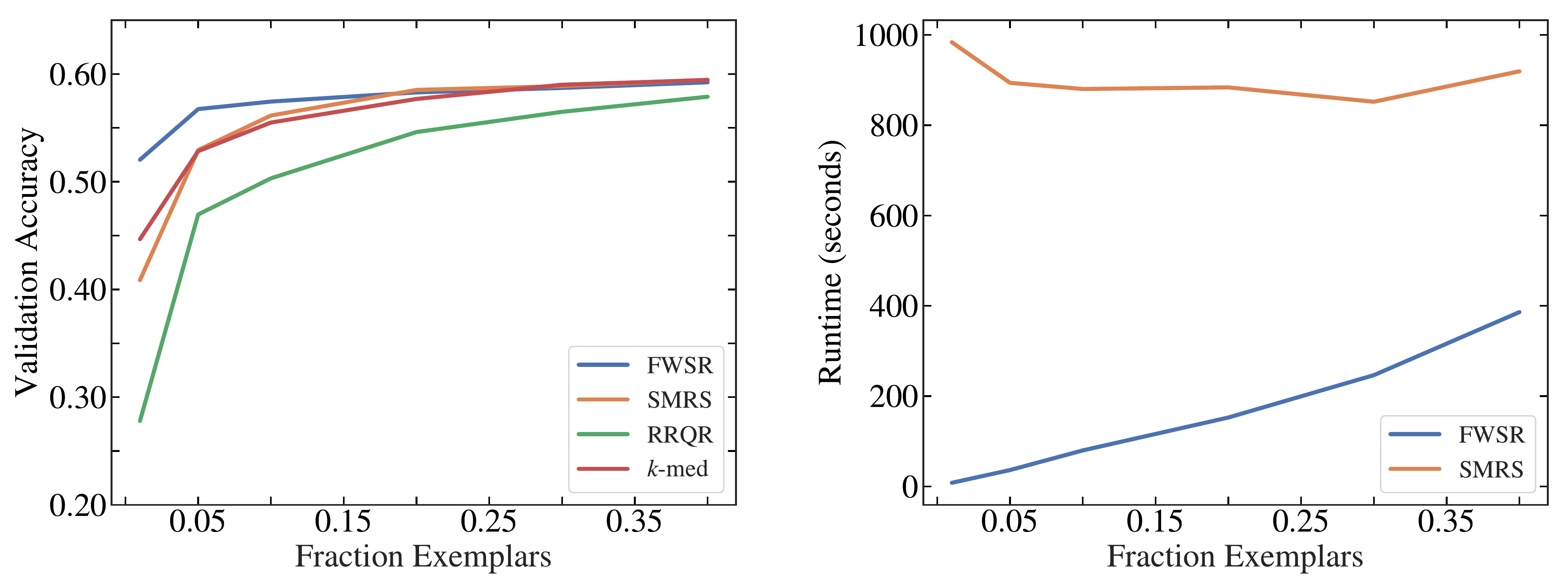}
      \caption{ \textbf{Experiment 3} Validation accuracy and run time versus number of exemplars for EMNIST. The EMNIST data set was subsampled such that each class had at most $1000$ data points so that SMRS could run in a reasonable time. Fraction exemplars denotes the number of exemplars as a percentage of the $1000$ data points in each class. Not displayed: the run time for RRQR and $k$-medoids was $\leq 5$ seconds along the abscissa.}
      \label{speed_graph}
\end{figure*}


We tested our algorithms on the Extended Yale Face Database B (E-YaleB), 20 Newsgroups (News20), Credit Fraud (Credit), and EMNIST ByClass (EMNIST) datasets; descriptions of each dataset can be found in Table \ref{tbl:datasets}. 
For the Credit Fraud training and validation data sets, we independently center and normalize along features (i.e., along each row of $A$) as a preprocessing step. This was left as a hyperparameter choice for the E-YaleB data set. It was not done for the News20 data set in order to preserve the sparsity in the data, and it was not done for E-MNIST because we empirically observed poor validation set performance.

We display the best cross validated accuracies in Table \ref{sup_acc}. We also display the total time it takes for each algorithm to find the exemplars and train a Linear SVM Model in Table \ref{runtime_table}.
While Table \ref{sup_acc} shows that FWSR is competitive and can outperform the other algorithms in different settings, Table \ref{runtime_table} shows that the algorithm also has a fast end-to-end training time. Note that for all the data sets, SMRS is the slowest algorithm while FWSR strikes a balance. 

In an effort to quantify the effect of $k$ on downstream accuracy, Figure \ref{speed_graph} plots the validation accuracy and run time for the SVM classifier on EMNIST against the number of exemplars selected.
Not only does FWSR outperform the other methods in terms of validation accuracy, but it also shows significant speed-ups compared to SMRS. This is consistent with the results presented in Tables \ref{sup_acc} and \ref{runtime_table}.

\begin{table}[H]
\resizebox{\columnwidth}{!}{
\begin{tabularx}{1.04\linewidth}{lcccc}
\toprule
&\textsc{E-YaleB} & \textsc{News20} & \textsc{Credit} & \textsc{EMNIST}   \\
\midrule
\textsc{All} & $5.42$ & $2.11$ & $0.27$ & $18133.05$ \\
\textsc{FWSR} & $1.29$ & $8.24$ & $7.51$ & $159.48$ \\
\textsc{$k$-med} & $8.21$  & $5.79$  & $0.22$ & $17.39$ \\
\textsc{RRQR} & $3.43$ & $70.16$ & $0.01$& $23.85$ \\
\textsc{SMRS} & $14.80$ & $3331.44$ & $220.56$ & \;\; -- \\
\bottomrule
\end{tabularx}}
\caption{\textbf{Experiment 3} Total reduction time and training time in seconds for an SVM across all the methods. Note that \textsc{All} has no reduction time and simply represents the training time of the SVM on the entire data set. Across 20 trials, the standard deviation of \textsc{$k$-med} was 10.592, 0.097, 0.011, and 0.102 seconds in order from \textsc{E-YaleB} to \textsc{EMNIST}.}
\label{runtime_table}
\end{table}



\section{Conclusion}
In this paper, we proposed Franke-Wolfe Sparse Representation, an algorithm for solving the auto-regressive version of dictionary learning that helps identify a subset of the data that efficiently describes the entire data set. We show that our method can be seen as a natural relaxation of the boolean version of the problem and show that using FWSR, we are able to cut down the per iteration cost of state of the art methods by a factor of $d$. Furthermore, we provide a linear convergence rate for our algorithm as well as an interpretable terminating condition. We employ our algorithm on a variety of data sets and show the computational gain as well as its performance against other exemplar finding algorithms.


\newpage
\bibliographystyle{plain}
\bibliography{bibtext2}

\clearpage
\appendix
\onecolumn
\gdef\thesection{Appendix \Alph{section}}

{\centering{{\LARGE\bfseries Supplementary material}}}

\section{Linear Convergence Proof} \label{apx:proof}
\subsection{Prerequisites}
We prove linear convergence for a not strongly convex $f(X)$ (i.e., when $A$ is not full column rank) for $\beta > n$. For this appendix, (just like the main paper) we implicitely augment a row $\eta \textbf{1}^\top$ to $A$ to accommodate the translational invariance penalty. We begin this proof by redefining some variables:
\begin{align*}
\hat{A} &\coloneqq \begin{bmatrix}
              A & & 0\\
              & \ddots &\\
              0 & & A
              \end{bmatrix} 
&w_t &\coloneqq \begin{bmatrix}
              X^{(1)T}_t\\
              \vdots\\
              X^{(n)T}_t
              \end{bmatrix}\\
\hat{K} &\coloneqq \hat{A}^\top \hat{A}
& \hat{1}_i &\coloneqq \begin{cases}
              1 & \text{if } i \mod (d+1) = 0\\
              0 & \text{o.w.}
              \end{cases}\\
s_t &\coloneqq \arg \min_s \langle \nabla f, s\rangle &d_t &\coloneqq s_t - w_t\\
 &= \arg \max_s \langle \hA s - \hA w, \hA \hat{1} - \hA w\rangle & &\\
 &= \arg \max_s (s- w)^\top \hK(\hat{1} - w) &&
\end{align*}
Observe that our objective and gradient using this notation can be rewritten as:
    \begin{align*}
      f(w_t) &\coloneqq \left\|AX_t - A\right\|_{F}^2 \\
      &= \|\hat{A} w_t - \hat{A} \hat{1}\|_{2}^2\\
      \nabla f(w_t) &\coloneqq  2\hat{K}(w- \hat{1}) 
    \end{align*}

The next cost as a function of the previous cost is:
\begin{align}
f(w_{t+1}) &=  f(w_{t} + \gamma d_t)\\
&= (w_t + \gamma d_t - \hat{1})K (w_t + \gamma d_t - \hat{1}) \\
&= f(w_t) + \gamma^2 d_t^T K d_t + 2\gamma (w_t - \hat{1})^T K d_t \label{nextstepcost}\\
\frac{\partial f}{\partial \gamma} &= 2\gamma d_t^T K d_t + 2 (w_t - \hat{1}) K d_t = 0\\
\gamma_t &= \frac{d_t^T K (\hat{1} - w_t)}{d_t^T K d_t}\\
f(w_{t+1}) &= f(w_t) - \frac{(d_t^T K (1 - w_t))^2}{d_t^T K d_t} \label{nextstepcostV2}
\end{align}
The last line is true due to lemma \ref{stepsize} which is introduced below.

We will also be using the following helpful lemmas:
\begin{lemma}
\label{stepsize}
For $\beta > n$, the optimal step size satisfies: $0 \leq \gamma_t \leq 1$ for all $t$
\end{lemma}
\begin{proof}
Suppose $\gamma_t < 0$, this implies that $(s_t - w_t)^T \hat{K} (\hat{1} - w_t) < 0$, but because $s$ maximizes the quantity (fact 1), it must be that:
\begin{gather*}
0 >(s_t - w_t)^T \hat{K} (\hat{1} - w_t) \geq (\hat{1} - w_t)^T \hat{K} (\hat{1} - w_t) \geq 0
\end{gather*}
which is a contradiction.

Suppose that $\gamma_t > 1$, this implies that $(s_t - w_t)^T \hat{K} (\hat{1} - w_t) > d^T_t \hat{K} d^T_t$ (fact 2). However, in using $\gamma = 1$ in 
equation \eqref{nextstepcost}:
\begin{align*}
0 &\leq f(w_{t+1}) = f(w_t) + 2 d_t^T \hK (w_t - \hat{1}) + d_t^T \hat{K} d_t\\
&< f(w_t) + d_t^T \hK (w_t - \hat {1})\\
&= (\hat{1} - w_t)^T \hK (\hat{1} - w_t) + d_t^T \hK (w_t - \hat{1}) \leq 0
\end{align*}
where the first inequality comes from fact 2, and the second inequality comes from fact 1. This is also a contradiction.
\end{proof}

\begin{lemma}
\label{relint}
For $\beta > n$, $\hA \hat{1}$ is in the relative interior of $\{\hA u \: : \: u \in \mathcal{M}_{q, \beta}\}$. Furthermore, there exists $r > 0$ such that \[\hA w + (\|\hA \hat{1} - \hA w\|_{2} + r) \frac{\hA \hat{1} - \hA w}{\|\hA \hat{1} - \hA w\|_{2}}\] for all $w$ is in the interior of the domain as well
\end{lemma}
\begin{proof}
The Open Mapping Theorem proves the first point because $\hat{A}$ is by definition surjective to the space spanned by $\hat{A}$ and by construction there exists an open set around $\hat 1$ in the $\mathcal{M}_{q, \beta}$. Given that the first point is true, then the second point arises from the fact that 
\begin{gather*}
\hA w + (\|\hA \hat{1} - \hA w\|_{2} + r) \frac{\hA \hat{1} - \hA w}{\|\hA \hat{1} - \hA w\|_{2}} = \hA \hat{1}  + r \frac{\hA \hat{1} - \hA w}{\|\hA \hat{1} - \hA w\|_{2}} 
\end{gather*}
Since, there must exist an open ball around $\hA \hat{1}$, there must exist an $r$ such the above is true.
\end{proof}

\begin{lemma}
\label{A6}
The logistic equation
\begin{gather*}
x_{n+1} = \alpha x_n (1 - x_n)
\end{gather*}
for $x_0, \alpha \in [0,1]$ satisfies
\begin{gather*}
\forall n\in \mathbb{N},\ x_n \leq \frac{x_0}{\alpha^{-n} + x_0 n}
\end{gather*}
\end{lemma}
\begin{proof}
This is Lemma A.6 from \citep{campbell17}. The proof of this lemma can be found in Appendix A of the aforementioned paper.
\end{proof}

\subsection{Main Proof}
Starting with equation \eqref{nextstepcostV2}:
\begin{align*}
f(w_{t+1}) &= f(w_t) - \frac{(d_t^T \hK (\hat{1} - w_t))^2}{d_t^T \hK d_t}\\
&= f(w_t)\left( 1 - \frac{1}{(w_t - \hat{1})^T\hK (w_t - \hat{1})} \frac{(d_t^T \hK (\hat{1} - w_t))^2}{d_t^T \hK d_t  } \right)\\
&= f(w_t)\left( 1 - \left[\frac{(\hA s_t - \hA w_t)^T (\hA \hat{1} - \hA w_t)}{\|\hA s_t-  \hA w_t\|_{2} \| \hA \hat{1} - \hA w_t\|_{2}}\right]^2 \right)
\end{align*}

Furthermore, observe that because $s_t = \arg \max_s \langle \hA s - \hA w, \hA \hat{1} - Aw\rangle$, replacing $\hA s$ with any other point is a lower bound, so using equation A.57, for some $r > 0$ we can replace $As$ with $\hA w + (\|\hA \hat{1} - \hA w\|_{2} + r) \frac{\hA \hat{1} - \hA w}{\|\hA \hat{1} - \hA w\|_{2}}$ using lemma \ref{relint}. Notice that this vector is in the range of $\hA$.
\begin{align*}
\left(\frac{\hA s_t - \hA w_t }{\|\hA s_t-  \hA w_t\|_{2} } \right)^T\frac{(\hA \hat{1} - \hA w_t)}{\| \hA \hat{1} - \hA w_t\|_{2}}
  &\geq \left(\frac{(\|\hA \hat{1} - \hA w\|_{2} + r) \frac{\hA \hat{1} - \hA w}{\|\hA \hat{1} - \hA w\|_{2}}}{\|\hA s-  \hA w\|_{2}}\right)^T \frac{\hA \hat{1} - \hA w}{\| \hA \hat{1} - \hA w\|_{2}} \\
&= \frac{\|\hA \hat{1} - \hA w\|_{2} + r}{\|\hA s-  \hA w\|_{2}} 
= \frac{\sqrt{f(w_t)} + r}{\|\hA s-  \hA w\|_{2}}
\geq \frac{\sqrt{f(w_t)} + r}{2\beta \sqrt{\lambda_{\max}(\hK)}} \\
&\geq \frac{\sqrt{f(w_t)} + r}{C 2\beta \sqrt{\lambda_{\max}(\hK)}}\\
\end{align*}
for some $C > 1$. This implies that:
\begin{align*}
f(w_{t+1}) &\leq f(w_t)\left( 1 -  \left(\frac{\sqrt{f(w_t)} + r}{2\beta C \sqrt{\lambda_{\max}(\hK)}}\right)^2\right)\\
&\leq f(w_t)\left( 1 -  \frac{f(w_t)}{4\beta^2 C^2 \lambda_{\max}(\hK)} - \frac{r^2}{4\beta^2 C^2 \lambda_{\max}(\hK)}\right)\\
&= f(w_t)\left( \nu^2 -  \frac{f(w_t)}{4\beta^2 C^2 \lambda_{\max}(\hK)} \right)
\end{align*}
where $\nu^2 \coloneqq  1 - \frac{r^2}{4\beta^2 C^2 \lambda_{\max}(\hK)}$

With this relationship we can derive the critical recursive relationship:
\begin{align}
f(w_{t+1}) &\leq f(w_t)\nu^2 \left(1 -  \frac{f(w_t)}{4\beta^2 C^2\lambda_{\max}(K)\nu^2 } \right)\\
x_{t+1}&\leq x_t\nu^2 \left(1 -  x_t\right) \\
x_t &\coloneqq  \frac{f(w_t)}{4\beta^2C^2 \lambda_{\max}(\hK)\nu^2 }\label{recurrance}
\end{align}

\begin{claim}
\begin{gather*}
0 \leq \nu^2 = 1 - \frac{r^2}{4 \beta^2 C^2\lambda_{\max}(\hK)} < 1
\end{gather*}
\end{claim}
\begin{proof}
Recall that $r$ is the magnitude of change in the direction of $\frac{\hA \hat{1} - \hA w}{\|\hA \hat{1} - \hA w\|_{2}}$. So this means we can always pick $r$ small enough such that the numerator of $\frac{r^2}{4 \beta^2 C^2\lambda_{\max}(\hK)}$ is smaller than the denominator and that the new point that we chose is still in the interior of the domain; thereby making the claim true. 
\end{proof}

\begin{claim}
For a sufficiently large $C$,
\begin{gather*}
0 \leq x_t \leq 1
\end{gather*}
\end{claim}
\begin{proof}
$x_t$ is non-negative because all components of the fraction are nonnegative.
\begin{align*}
\frac{f(w_{t+1})}{4\beta^2 C^2\lambda_{\max}(\hK)\nu^2 } &= \frac{\|\hA (w - \hat{1})\|_{2}^2}{4 \beta^2 C^2 \lambda_{\max}(\hK)\nu^2 }
\leq \frac{1}{C^2 \nu^2} 
\end{align*}
We can always pick a $C$ large enough such that the quantity is less than $1$.
\end{proof}

Then finally in using lemma \ref{A6} with $x_t$ from \eqref{recurrance}, the proof is complete. The convergence rate is:
\begin{gather*}
x_t \leq \frac{x_0}{\nu^{-2t} + x_0 t}\\
\\
\frac{f(w_{t})}{4\beta^2C^2 \lambda_{\max}(\hK)\nu^2 } \leq \frac{\frac{f(w_{0})}{4\beta^2 C^2 \lambda_{\max}(\hK)\nu^2 }}{\nu^{-2t} + \frac{f(w_{0})}{4\beta^2C^2 \lambda_{\max}(\hK)\nu^2 } t} 
\leq \frac{C^{-2}\nu^{-2}}{\nu^{-2t} + C^{-2}\nu^{-2}t}
= \frac{1}{C^2 \nu^{2 - 2t} + t}
\end{gather*}
We can upper bound $f(w_0) \leq 4\beta^2 \lambda_{\max}(\hat{K})$ because $\frac{a}{at + b} = \frac{1}{t} \frac{a}{a + b/t}$ and $\frac{a}{a + c}$ is monotonically increasing in $a$ for all $a, c \geq 0$.

This implies for sufficiently large $t$ the convergence rate is linear:
\begin{align*}
f(w_t) &= \left\|AX_t - A\right\|_{F}^2 
\leq \frac{4\beta^2C^2 \lambda_{\max}(\hat{K})\nu^2 }{C^2 \nu^{2 - 2t} + t} \leq 2\beta^2\lambda_{\max}(\hat{K}) \nu^{2t}
\end{align*}

\end{document}